%% file: eccv2020submission.tex
\documentclass[runningheads]{llncs}
\usepackage{graphicx}
\usepackage{caption}
\usepackage{comment}
\usepackage{hyperref}
\usepackage{amsmath,amssymb} 
\usepackage{color}
\usepackage{wasysym}
\usepackage{multirow}
\usepackage{subcaption}

\usepackage{booktabs}
\usepackage{wasysym}
\newtheorem{mypro}{Proposition}
\newcommand{\etal}{\textit{et al}. }
\newcommand{\ie}{\textit{i}.\textit{e}. }
\newcommand{\eg}{\textit{e}.\textit{g}. }

\newcommand{\pxj}[1]{{\color{black}{#1}}}

\begin{document}
\pagestyle{headings}
\mainmatter

\title{Suppressing Mislabeled Data \\via Grouping and Self-Attention} 



\titlerunning{Suppressing Mislabeled Data \\via Grouping and Self-Attention}
%
\author{Xiaojiang Peng$^\star$ \inst{1,2} \and Kai Wang\thanks{Equally-contributed first authors.$^\dagger$Corresponding author (yu.qiao@siat.ac.cn)}\inst{1,2} \and
Zhaoyang Zeng$^\star$\inst{3} \and
Qing Li \inst{4} \and
Jianfei Yang \inst{5} \and
Yu Qiao$^\dagger$\inst{1,2}}
\authorrunning{X. Peng, K. Wang, Z. Zeng, Q. Li, J. Yang and Y. Qiao}
%
\institute{Guangdong-Hong Kong-Macao Joint Laboratory of Human-Machine Intelligence-Synergy Systems, Shenzhen Institutes of Advanced Technology, Chinese Academy of Sciences, 518055, China \and
SIAT Branch, Shenzhen Institute of Artificial Intelligence and Robotics for Society \and
Sun Yat-sen University \and Southwest Jiaotong University \and
Nanyang Technological University, Singapore}
\maketitle

\begin{abstract}
Deep networks achieve excellent results on large-scale clean data but degrade significantly when learning from noisy labels. To suppressing the impact of mislabeled data, this paper proposes a conceptually simple yet efficient training block, termed as Attentive Feature Mixup (AFM), which allows paying more attention to clean samples and less to mislabeled ones via sample interactions in small groups. Specifically, this plug-and-play AFM first leverages a \textit{group-to-attend} module to construct groups and assign attention weights for group-wise samples, and then uses a \textit{mixup} module with the attention weights to interpolate massive noisy-suppressed samples. The AFM has several appealing benefits for noise-robust deep learning. (i) It does not rely on any assumptions and extra clean subset. (ii) With massive interpolations, the ratio of useless samples is reduced dramatically compared to the original noisy ratio. (iii) \pxj{It jointly optimizes the interpolation weights with classifiers, suppressing the influence of mislabeled data via low attention weights. (iv) It partially inherits the vicinal risk minimization of mixup to alleviate over-fitting while improves it by sampling fewer feature-target vectors around mislabeled data from the mixup vicinal distribution.} Extensive experiments demonstrate that AFM yields state-of-the-art results on two challenging real-world noisy datasets: Food101N and Clothing1M. The code will be available at \href{https://github.com/kaiwang960112/AFM}{https://github.com/kaiwang960112/AFM}.
\end{abstract}


\section{Introduction}
\label{intro}
\input{introduction}

\section{Related Work}
\label{related}
\input{relatedwork}

\section{Attentive Feature Mixup}
\label{method}

\input{method}

\section{Experiments}
\label{experiments}
\input{experiments}

\section{Conclusion}
This paper proposed a conceptually simple yet efficient training block, termed as Attentive Feature Mixup (AFM), to address the problem of learning with noisy labeled data. Specifically, AFM is a plug-and-play training block, which mainly leverages grouping and self-attention to suppress mislabeled data and does not rely on any assumptions and extra clean subset. We conducted extensive experiments on two challenging real-world noisy datasets: Food101N and Clothing1M. Quantitative and qualitive results demonstrated that our AFM is superior to recent state-of-the-art methods. In addition, the grouping and self-attention is expected to extend in other topics, \eg semi-supervised learning, where one may conduct this module for real annotations and pseudo labels to automatically suppress incorrect pseudo labels.

\textbf{Acknowledge}.
This work is partially supported by National Key Research and Development Program of China (No. 2020YFC2004800), National Natural Science Foundation of China (U1813218, U1713208), Science and Technology Service Network Initiative of Chinese Academy of Sciences (KFJ-STS-QYZX-092), Guangdong Special Support Program (2016TX03X276), and  Shenzhen Basic Research Program (JSGG20180507182100698, CXB201104220032A), Shenzhen Institute of Artificial Intelligence and Robotics for Society.


\bibliographystyle{splncs04}
\bibliography{egbib}
\end{document}

%% file: introduction.tex
In recent years, deep neural networks (DNNs) have achieved great success in various tasks, particularly in supervised learning tasks on large-scale image recognition challenges, such as ImageNet \cite{deng2009imagenet} and COCO \cite{lin2014microsoft}. One key factor that drives impressive results is the large amount of well-labeled images. However, high-quality annotations are laborious and expensive, even not always available in some domains. To address this issue, an alternative solution is to crawl a large number of web images with tags or keywords as annotations~\cite{gong2014multi,li2017webvision}. These annotations provide weak supervision, which are noisy but easy to obtain.

\begin{figure}[tp]
\center
	\includegraphics[width=0.7\textwidth]{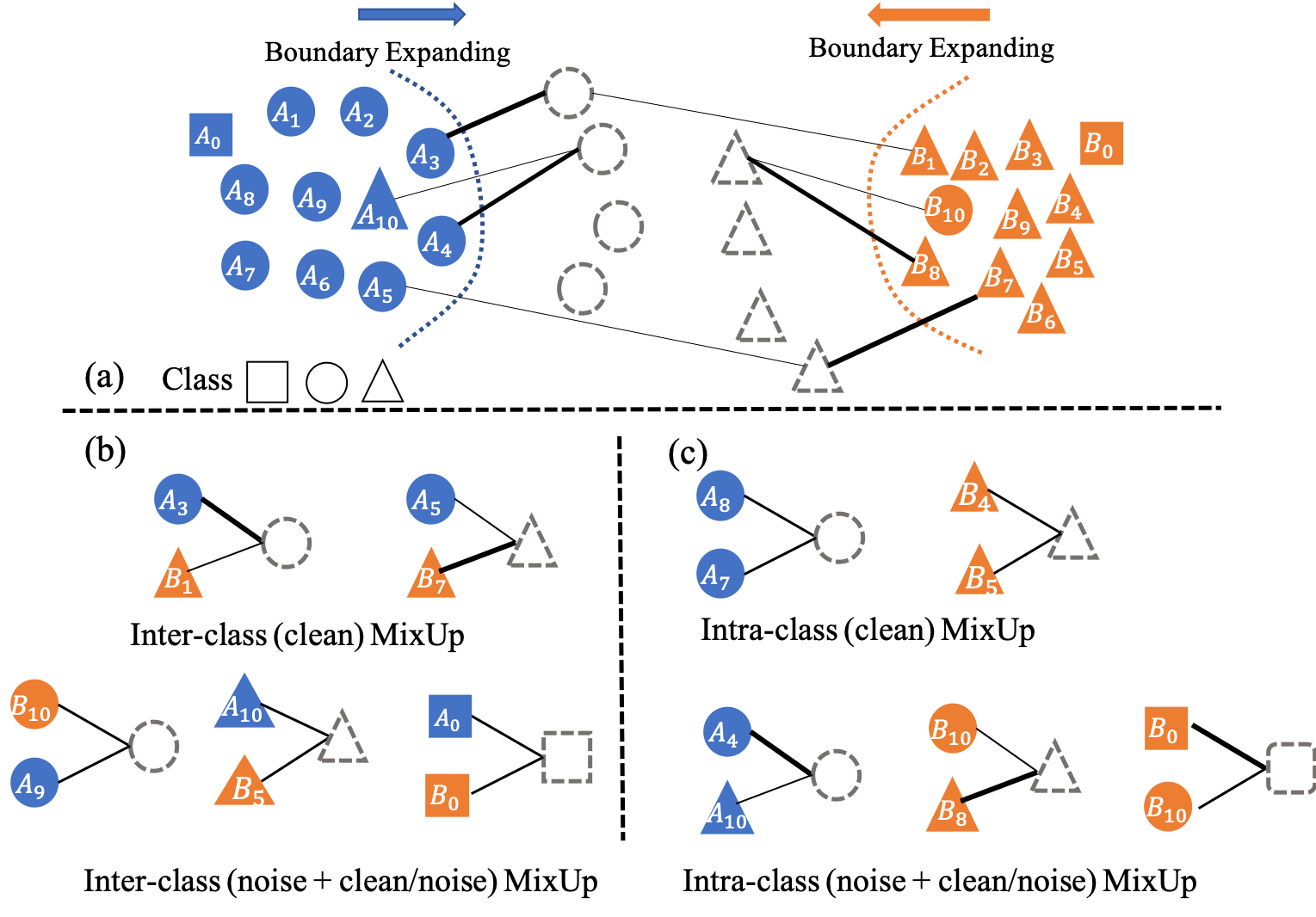}
	\caption{\pxj{Suppressing mislabeled samples by grouping and self-attention mixup. \textit{Different colors and shapes denote given labels and ground truths. Thick and thin lines denote high and low attention weights, respectively.} $A_0, A_{10}, B_0$, and $B_{10}$ are supposed to be mislabeled samples, and can be suppressed by assigning low interpolation weights in mixup operation.}}
	\label{fig:simple_motivation}
\end{figure}
 
In general, noisy labeled examples hurt generalization because DNNs easily overfit to noisy labels~\cite{frenay2014classification,sukhbaatar2014training,arpit2017closer}. To address this problem, it is intuitive to develop noise-cleaning methods which aim to correct the mislabeled samples either by joint optimization of classification and relabeling~\cite{tanaka2018joint} or by iterative self-learning~\cite{han2019deep}. However, the noise-cleaning methods often suffer
from the main difficulty in distinguishing mislabeled samples from hard samples.
Another solution is to develop noise-robust methods which aims to reduce the contributions of mislabeled samples for model optimization. 
Along this solution, some methods estimate a matrix for label noise modeling and use it to adapt output probabilities and loss values~\cite{sukhbaatar2014training,xiao2015learning,patrini2017making}. Some others resort to curriculum learning~\cite{bengio2009curriculum} by either designing a step-wise easy-to-hard strategy for training ~\cite{guo2018curriculumnet} or introducing an extra MentorNet~\cite{jiang2017mentornet} for sample weighting. However, these methods independently estimate the importance weights for individuals which ignore the comparisons among different samples while they have been proven to be the key of humans to perceive and learn novel concepts from noisy input images~\cite{schmidt1992new}.
Some other solutions follow semi-supervised configuration where they assume a small manually-verified set can be used~\cite{li2017learning,veit2017learning,lee2017cleannet,Li_2019_CVPR}. However, this assumption may be not supported in real-world applications.
\pxj{With the Vicinal Risk Minimization(VRM) principle, mixup~\cite{zhang2017mixup,verma2018manifold} exploits a vicinal distribution for sampling virtual sample-target vectors, and proves its  robustness  
for synthetic noisy data. But its effectiveness is limited in real-world noisy data~\cite{arazo2019unsupervised}.} 

In this paper, we propose a conceptually simple yet efficient training block, termed as Attentive Feature Mixup (AFM), to suppress mislabeled data thus to make training robust to noisy labels. The AFM is a plug-and-play block for training any networks and is comprised of two crucial parts: 1) a \textit{Group-to-Attend} (GA) module that first randomly groups a minibatch images into small subsets and then estimates sample weights within those subsets by an attention mechanism, and 2) a \textit{mixup} module that interpolates new features and soft labels according to self-attention weights. Particularly, for the GA module, we evaluate three feature interactions to estimate group-wise attention weights, namely concatenation, summary, and element-wise multiplication. The interpolated samples and original samples are respectively fed into an interpolation classifier and a normal classifier. Figure \ref{fig:simple_motivation} illustrates how AFM suppress mislabeled data. Generally, there exists two main types of mixup: intra-class mixup (Figure \ref{fig:simple_motivation} (c)) and inter-class mixup (Figure \ref{fig:simple_motivation} (b)).
For both types, the interpolations between mislabeled samples and clean samples may become useful for training with adaptive attention weights, \ie low weights for the mislabeled samples and high weights for the clean samples.
In other words, our AFM hallucinates numerous useful noisy-reduced samples to guide deep networks learn better representations from noisy labels.  
 Overall, as a noisy-robust training method, our AFM is promising in the following aspects.
\begin{itemize}
\item It does not rely on any assumptions and extra clean subset.
\item With AFM, the ratio of harmful noisy interpolations (\ie between noisy samples) over all interpolations is largely less than the original noisy ratio.
\item It jointly optimizes the mixup interpolation weights and classifier, suppressing the influence of mislabeled data via low attention weights. 
\item It partially inherits the vicinal risk minimization of mixup to alleviate over-fitting while improves it by sampling less feature-target vectors around mislabeled data from the mixup vicinal distribution. 
\end{itemize} 
We validate our AFM on two popular real-world noisy-labeled datasets: Food101N \cite{lee2017cleannet} and Clothing1M \cite{xiao2015learning}. Experiments show that our AFM outperforms recent state-of-the-art methods significantly with accuracies of \textbf{87.23}\% on Food101N and \textbf{82.09}\% on Clothing1M.

%% file: relatedwork.tex
\subsection{Learning with Noisy Labeled Data}
Learning with noisy data has been vastly studied on the literature of machine learning and computer vision. 
Methods on learning with label noise can be roughly grouped into three categories: noise-cleaning methods, semi-supervised methods and noise-robust methods.


First, noise-cleansing methods aim to identify and remove or relabel noisy samples with filter approaches~\cite{barandela2000decontamination,miranda2009use}.
Brodley \textit{et al.}~\cite{brodley1999identifying} propose to filter noisy samples using ensemble classifiers with majority and consensus voting. Sukhbaatar \textit{et al.}~\cite{sukhbaatar2014training} introduce an extra noise layer into a standard CNN which adapts the network outputs to match the noisy label distribution. Daiki \textit{et al.}~\cite{tanaka2018joint} propose a joint optimization framework to train deep CNNs with label noise, which updates the network parameters and labels alternatively.
Based on the consistency of the noisy groundtruth and the current prediction of the model, Reed \textit{et al.}~\cite{reed2014training} present a `Soft' and a `Hard' bootstrapping approach to relabel noisy data. Li \textit{et al.}~\cite{li2017learning} relabel noisy data using the noisy groundtruth and the current prediction adjusted by a knowledge graph constructed from DBpedia-Wikipedia.

Second, semi-supervised methods aim to improve performance using a small manually-verified clean set. Lee \textit{et al.}~\cite{lee2017cleannet} train an auxiliary CleanNet to detect label noise and adjust the final sample loss weights. In the training process, the CleanNet needs to access both the original noisy labels and the manually-verified labels of the clean set.
Veit \textit{et al.}~\cite{veit2017learning} use the clean set to train a label cleaning network but with a different architecture. These methods assume there exists such a label mapping from noisy labels to clean labels. Xiao \textit{et al.}~\cite{xiao2015learning} mix the clean set and noisy set, and train an extra CNN and a classification CNN to estimate the posterior distribution of the true label. Li \textit{et al.}~\cite{Li2020ProductIR} first train a teacher model on clean and noisy data, and then distill it into a student model trained on clean data.

Third, the noise-robust learning methods are assumed to be not too sensitive to the presence of label noise, which directly learn models from the noisy labeled data~\cite{joulin2016learning,krause2016unreasonable,misra2016seeing,patrini2017making,Wang_2020_CVPR}.
Manwani \textit{et al.}~\cite{manwani2013noise} present a noise-tolerance algorithm under the assumption that the corrupted probability of an example is a function of the feature vector of the example. 
With synthetic noisy labeled data, Rolnick \textit{et al.}~\cite{rolnick2017deep} demonstrate that deep learning is robust to noise when training data is sufficiently large with large batch size and proper learning rate. Guo \textit{et al.}~\cite{guo2018curriculumnet} develop a curriculum training scheme to learn noisy data from easy to hard. Han \etal~\cite{han2019deep} propose a Self-Learning with Multi-Prototype (SMP) method to learn robust features via alternatively training and clustering which is time-consuming. Wang \textit{et al.}~\cite{Wang_2020_CVPR} propose to suppress uncertain samples with self-attention, ranking loss, and relabeling.
Our method is most related to MetaCleaner \cite{zhang2019metacleaner}, which hallucinates a clean (precisely noise-reduced) representation by mixing samples (the ratio of the noisy images need to be small) from the same category. Our work differs from it in that i) we formulate the insight as attentive mixup, and ii) hallucinate noisy-reduced samples not only within class but also between classes which significantly increases the number of interpolations and expands the decision boundaries. Moreover, we introduce more sample interactions rather than the concatenation in \cite{zhang2019metacleaner}, and find a better one.   

\subsection{Mixup and Variations}
Mixup \cite{zhang2017mixup} regularizes the neural network to favor simple linear behavior in-between training examples. Manifold Mixup \cite{verma2018manifold} leverages semantic interpolations in random layers as additional training signal to optimize neural networks. The interpolation weights of those two methods are drawn from a $\beta$ distribution. Meta-Mixup \cite{mai2019metamixup} introduces a meta-learning based online optimization approach to dynamically learn the interpolation policy from a reference set. AdaMixup \cite{guo2019mixup} also learns the interpolation policy from dataset with an additional network to infer the policy and an intrusion discriminator. Our work differs from these variations in that i) we design a Group-to-Attend mechanism to learn attention weights for interpolating in a group-wise manner which is the key to reduce the influence of noises and ii) we address the noisy-robust problem on real-world noisy data and achieve state-of-the-art performance. 

%% file: method.tex
 As proven in cognitive studies, we human mainly perceive and learn novel concepts from noisy input images by comparing and summary \cite{schmidt1992new}. Based on this motivation, we propose a simple yet efficient model, called Attentive Feature Mixup (AFM), which aims to learn better features by making clean and noisy samples interact with each other in small groups. 

\subsection{Overview}
Our AFM works on traditional CNN backbones and includes two modules: i) Group-to-Attend (GA) module and ii) mixup module, as shown in Figure \ref{fig:pipeline}.

 Let $\mathcal{B}=\{(\mathbf{I}_1,y_1), (\mathbf{I}_2,{y}_2), \cdots, (\mathbf{I}_n,{y}_n)\}$ be the mini-batch set of a noisy labeled dataset, which contains $n$ samples, and ${y}_i \in \mathcal{R}^C$ is the noisy one-hot label vector of image $\mathbf{I}_i$. The AFM works as the following procedure. First, a CNN backbone $\phi(\cdot;\theta)$ with parameter $\theta$ is used to extract image features $\{{x}_1, {x}_2, \cdots, {x}_n\}$.
Then, the Group-to-Attend (GA) module is used to divide the mini-batch images into small groups and learn attention weights for each samples within each group. Subsequently, with the group-wise attention weights, a mixup module is used to interpolate new samples and soft labels. Finally, these interpolations along with the original image features are fed into an interpolation classifier $f_{c1}$ (\ie FC layer) and a normal classifier $f_{c2}$ (\ie FC layer), respectively. Particularly, the interpolation classifier is supervised by the soft labels from the mixup module and the normal classifier by the original given labels which are noisy. Our AFM partially inherits the vicinal risk minimization of mixup to alleviate over-fitting with massive interpolations. Further, with jointly optimizing the mixup interpolation weights and classifier, AFM  improves mixup by sampling less feature-target vectors around mislabeled data from the mixup vicinal distribution.

\begin{figure}[t]
	\includegraphics[width=0.95\textwidth]{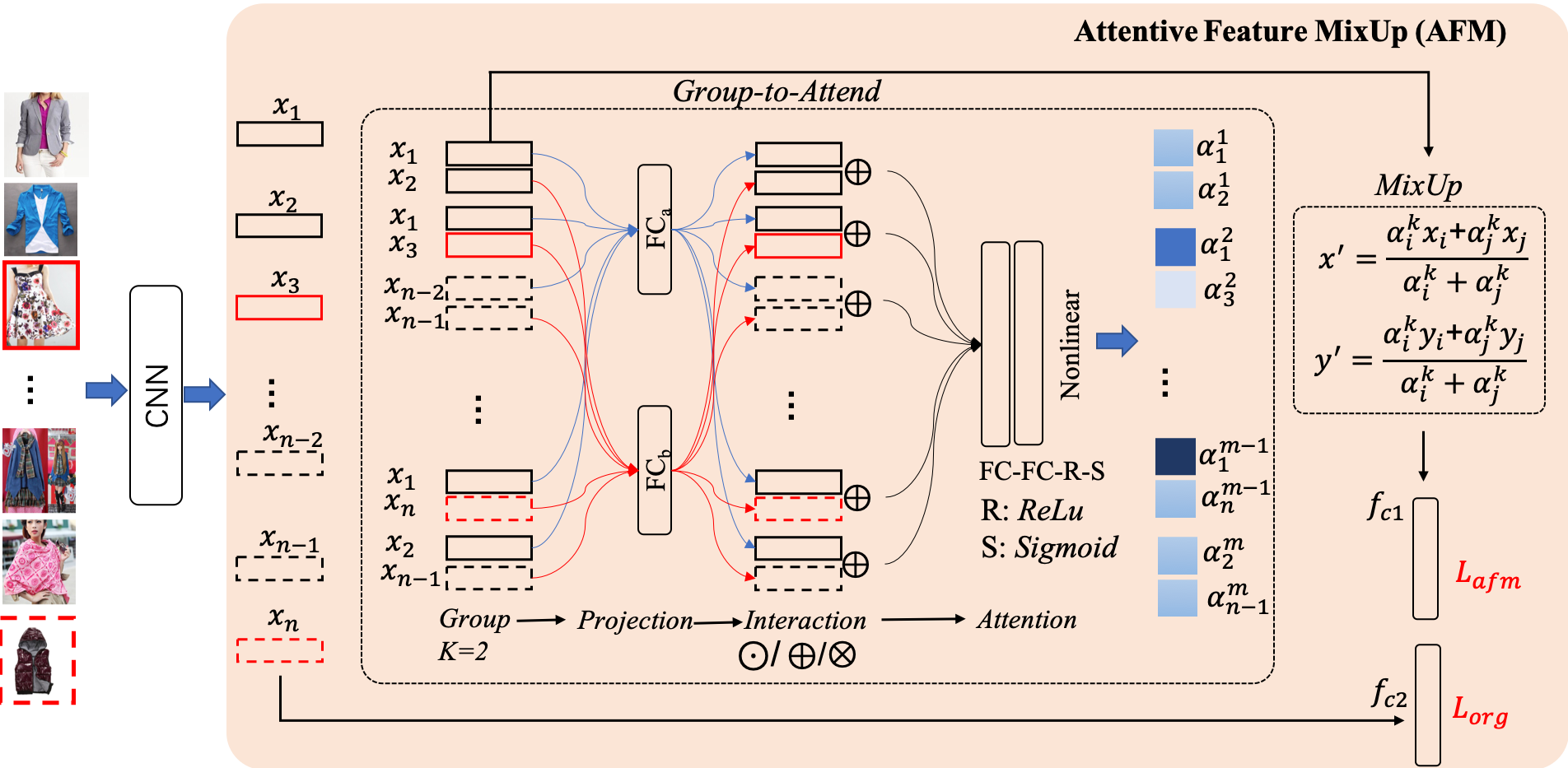}
	\caption{The pipeline of Attentive Feature Mixup (AFM). Given a mini-batch of $n$ images, a backbone CNN is first applied for feature extraction. Then, a Group-to-Attend (GA) module randomly composites massive groups with the group size $K$ and linearly projects each element within a group with a separated FC layer, and then combines each group with an interaction (\ie concatenation ($\odot$), sum ($\oplus$), and element-wise multiplication ($\otimes$) ), and finally outputs $K$ attention weights for each group. With the group-wise attention weights, a mixup module is used to interpolate new samples and soft labels.}
	\label{fig:pipeline}
\end{figure}


\subsection{Group-to-Attend Module}

In order to obtain meaningful attention weights, \ie high weights for clean samples and low weights for mislabeled samples, we elaborately design a Group-to-Attend module, which consists of four crucial steps. First, we randomly and repeatedly selecting $K$ samples to construct groups as many as possible (the number of groups depends on the input batch size and the GPU memory). Second, we use $K$ fully-connected (FC) layers to map the ordered samples of each group into new feature embeddings for sample interactions. As an example of $K=2$ in Figure \ref{fig:pipeline}, $x_i$ and $x_j$ are linearly projected as,
\begin{equation}
\tilde{x}_i = f_a(x_i;w_a), ~~ \tilde{x}_j = f_b(x_j;w_b),
\end{equation}
where $w_a$ and $w_b$ are the parameters of FC layer $f_a$ and $f_b$, respectively. Third, we further make $\tilde{x}_i$ and $\tilde{x}_j$  interact for group-wise weight learning.  Specifically, we experimentally explore three kinds of interactions: concatenation ($\odot$), sum ($\oplus$), and element-wise multiplication ($\otimes$).
Last, we apply a light-weight self-attention network to estimate group-wise attention weights. Formally, for $K=2$ and the sum interaction, this step can be defined as follows,
\begin{eqnarray}
[\alpha_i^k, \alpha_j^k] & = & \psi_{t}(\tilde{x}_i\oplus \tilde{x}_j;\theta_t) \nonumber  \\ 
&=& \psi_{t}(f_a(x_i;w_a)\oplus f_b(x_j;w_b);\theta_t),
\label{eq:att}
\end{eqnarray}
where $\psi_{t}$ is the attention network, $\theta_t$ denotes its parameters, and $k$ denotes the $k-$th group. For the architecture of $\psi_{t}$, we follow the best one of \cite{zhang2019metacleaner}, \ie FC-FC-ReLu-Sigmoid. It is worth noting that feature interaction is crucial for learning meaningful attention weights since the relationship between noisy and clean samples within a group can be learned efficiently while not the case of non-interaction (\ie learning weights for each other separately).
\begin{mypro}
The attention weights are meaningful with sum interaction if and only if $f_a\neq f_b$.
\end{mypro}
\begin{proof}
Assume we remove the projection layers $f_a$ and $f_b$ or share them as the same function $f$, then Eq. (\ref{eq:att}) is rewritten as, 
\begin{eqnarray}
[\alpha_i^k, \alpha_j^k] &=& \psi_{t}(f({x}_i)\oplus f({x}_j);\theta_t) \nonumber\\
&=& \psi_{t}(f({x}_j)\oplus f({x}_i);\theta_t). 
\end{eqnarray}
As can be seen, removing or sharing the projection makes the attention network $\psi_{t}$ confirm the commutative law of addition. This corrupts the attention weights to be random since an attention weight can correspond to both samples for the following MixUp module. 
\end{proof}

\textbf{The effect of GA}.
An appealing benefit of our GA is that it reduces the impact of noisy-labeled samples significantly. Let $N_{noisy}$ and $N_{total}$ represent the number of the noisy images and total images in a noisy dataset, respectively. The noise ratio is $\frac{N_{noisy}}{N_{total}}$ in the image-wise case. Nevertheless, the number of pure noisy groups (\ie all the images are mislabeled in these groups) in the group-wise case becomes $A_{N_{noisy}}^K$. With $K=2$, we have,
\begin{equation}
\frac{N_{noisy}}{N_{total}} > \frac{A_{N_{noisy}}^2}{A_{N_{total}}^2} = \frac{N_{noisy} (N_{noisy} - 1)}{N_{total} (N_{total} - 1)} \approx \frac{N_{noisy}^2}{N_{total}^2}.
\end{equation}
We argue that GA can reduce the pure noisy ratio dramatically and partial noisy groups (\ie some images within these groups are corrected-labeled) may provides useful supervision by the well-trained attention network. However, though the ratio is smaller when $K$ becomes larger, large $K$ may lead to over-smooth features for the new interpolations which are harmful for discriminative feature learning.
%

\subsection{Mixup Module}
The mixup module interpolates virtual feature-target vectors for training. Specifically, following classic mixup vicinal distribution, we  normalize the attention weights into range [0, 1]. Formally, for $K=2$ and group members \{$x_i$, $x_j$\}, the mixup can be written as follows,
\begin{eqnarray}
x' &=& \frac{1}{\sum_\alpha}(\alpha_{i}x_i + \alpha_{j}x_j), \\
y' &=& \frac{1}{\sum_\alpha}(\alpha_{i}y_i + \alpha_{j}y_j),
\end{eqnarray}
where $x'$ and $y'$ are the interpolated feature and soft label.

\subsection{Training and Inference}
\textbf{Training}. Our AFM along with the CNN backbone can be trained in an end-to-end manner. Specifically, we conduct a multi-task training scheme to separate the contributions of original training samples and new interpolations. Let $f_{c1}$ and $f_{c2}$ respectively denote the classifiers (include the Softmax or Sigmoid operations) of interpolations and original samples, we can formulate the training loss in a mini batch as follows,
\begin{eqnarray}
 \mathcal{L}_{total} & = & \lambda \mathcal{L}_{afm} + (1-\lambda) \mathcal{L}_{org}  \\
&=&  \frac{\lambda}{m}\sum_{i=1}^m \mathcal{L}(f_{c1}( x'_i), y'_i) + \frac{(1-\lambda)}{n}\sum_{i=1}^n \mathcal{L}(f_{c2}( x_i), y_i), \nonumber
\end{eqnarray}
where $n$ is the batch size, $m$ is the number of interpolations, and $\lambda$ is a trade-off weight. We use the Cross-Entropy loss function for both ${L}_{afm}$ and ${L}_{org}$. In this way, our AFM can be viewed as a regularizer over the training data by massive interpolations. As proven in \cite{zhang2017mixup,verma2018manifold}, this regularizer can largely improve the generalization of deep networks. In addition, the parameters of $f_{c1}$ and $f_{c2}$ can be shared since both original features and interpolations are in same dimensions. 

\textbf{Inference}. After training, both the GA module and mixup module can be simply removed since we do not need to compose new samples at test stage. We keep the classifiers $f_{c1}$ and $f_{c2}$ for inference. Particularly, they are identical and we can conduct inference as traditional CNNs if the parameters are shared.

%% file: experiments.tex

In this section, we first introduce datasets and implementation details, and then compare our AFM with the state-of-the-art methods. Finally, we conduct ablation studies with qualitative and quantitative results.

\subsection{Datasets and Implementation Details}
In this paper, we conduct experiments on two popular real-world noisy datasets: Food101N \cite{lee2018cleannet} and Clothing1M \cite{xiao2015learning}. 
 \textbf{Food101N} consists of 365k images that are crawled from Google, Bing, Yelp, and TripAdvisor using the Food-101 taxonomy. The annotation accuracy is about 80\%. The clean dataset Food-101 is collected from \textit{foodspotting.com} which contains 101 food categories with 101,000 real-world food images totally. For each class, 750 images are used for training, the other 250 images for testing. In our experiments, following the common setting, we use  all images of Food-101N as the noisy dataset, and report the overall accuracy on the Food-101 test set. 
 \textbf{Clothing1M} contains 1 million images of clothes with 14 categories. Since most of the labels are generated by the surrounding text of the images on the Web, a large amount of annotation noises exist, leading to a low annotation accuracy of 61.54\% \cite{xiao2015learning}. The human-annotated set of Clothing1M is used as the clean set which is officially divided into training, validation and testing sets, containing 50k, 14k and 10k images respectively. We report the overall accuracy on the clean test set of Clothing1M.


\textbf{Implementation Details}
As widely used in existing works, ResNet50 is used as our CNN backbone and initialized by the official ImageNet pre-trained model. For each image, we resize the image with a short edge of 256 and random crop $224\times224$ patch for training. We use SGD optimizer with a momentum of 0.9. The weight decay is 5 $\times$ $10^{-3}$, and the batch size is 128. For Food101N, the initial learning rate is 0.001 and divided by 10 every 10 epochs. We stop training after 30 epochs. For Clothing1M, the initial learning rate is 0.001 and divided by 10 every 5 epochs. We stop training after 15 epochs. All the experiments are implemented by Pytorch with 4 NVIDIA V100 GPUs. The default $\lambda$ and $K$ are 0.75 and 2, respectively. By default, the classifiers $f_{c1}$ and $f_{c2}$ are shared.

\begin{table}[t]
\caption{Comparison with the state-of-the-art methods on Food101N dataset. VF(55k) is the noise-verification set used in CleanNet \cite{lee2018cleannet}. }
\begin{center}
\begin{tabular}{@{}lcccccccc@{}}
\toprule
 Method & Training Data & Training time & Acc \\ \midrule
Softmax \cite{lee2018cleannet} & Food101 &-- & 81.67 \\
Softmax \cite{lee2018cleannet} & Food101N &-- &81.44 \\
Weakly Supervised \cite{zhuang2017attend} & Food101N &--& 83.43 \\
CleanNet($w_{hard}$) \cite{lee2018cleannet} & Food101N + VF(55K) &-- &83.47 \\
CleanNet($w_{soft}$) \cite{lee2018cleannet} & Food101N + VF(55K) &--& 83.95 \\
MetaCleaner \cite{zhang2019metacleaner} & Food101N &-- &85.05 \\
SMP \cite{han2019deep} & Food101N & --&85.11\\
ResNet50 (baseline) & Food101N & 4h16min40s & 84.51 \\ \midrule
\textbf{AFM (Ours)} & Food101N & 4h17min4s& \textbf{87.23} \\
\bottomrule
\end{tabular}
\end{center}
\label{tab:Food101N}
\end{table}
 
\subsection{Comparison on Food101N }
We compare AFM to the baseline model and existing state-of-the-art methods in Table \ref{tab:Food101N}. AFM improves our strong baseline from 84.51\% to 87.23\%, and consistently outperforms recent state-of-the-art methods with large margins. Moreover, our AFM is almost free since it only increases training time by 24s. Specifically, AFM outperforms \cite{zhuang2017attend} by 3.80\%, CleanNet($w_{soft}$) by 3.28\%, and SMP \cite{han2019deep} by 2.12\%. We notice that, CleanNet($w_{hard}$) and CleanNet($w_{soft}$) use extra 55k manually-verified images, while we do not use any extra images. In particular, MetaCleaner \cite{zhang2019metacleaner} uses a similar scheme but limited in intra-class mixup and its single feature interaction type, which leads to 2.18\% worse than our AFM. An ablation study will further discuss these issues in the following section. 

\begin{table}[tp]
\caption{Comparison with the state-of-the-art methods on Clothing1M. VF(25k) is the noise-verification set used in CleanNet \cite{lee2018cleannet}. }
\begin{center}
\begin{tabular}{@{}lccccccc@{}}
\toprule
Method & Training Data  & Acc. (\%)\\ \midrule
Softmax \cite{lee2018cleannet} & 1M noisy & 68.94 \\
Weakly Supervised \cite{zhuang2017attend} & 1M noisy & 71.36 \\
JointOptim \cite{zhang2019metacleaner} & 1M noisy & 72.23 \\
MetaCleaner \cite{zhang2019metacleaner} & 1M noisy & 72.50 \\
SMP (Final)\cite{han2019deep} & 1M noisy & \textbf{74.45} \\
SMP (Initial) \cite{han2019deep} & 1M noisy & 72.09 \\
\textbf{AFM (Ours)} & 1M noisy & \textbf{74.12} \\
\midrule
CleanNet($w_{hard}$) \cite{lee2018cleannet} & 1M noisy + VF(25K) & 74.15 \\
CleanNet($w_{soft}$) \cite{lee2018cleannet} & 1M noisy + VF(25K) & 74.69 \\
MetaCleaner \cite{zhang2019metacleaner} & 1M noisy + VF(25K) & 76.00 \\
SMP \cite{han2019deep} & 1M noisy + VF(25K) & 76.44\\
\textbf{AFM (Ours)} & 1M noisy + VF(25K) & \textbf{77.21} \\
\midrule
CleanNet($w_{soft}$) \cite{lee2018cleannet} & 1M noisy + Clean(50K) & 79.90 \\
MetaCleaner \cite{zhang2019metacleaner} & 1M noisy + Clean(50K) & 80.78 \\
SMP \cite{han2019deep} & 1M noisy + Clean(50K) & 81.16\\
CurriculumNet \cite{guo2018curriculumnet} & 1M noisy + Clean(50K) & 81.50 \\
\textbf{AFM (Ours)}  & 1M noisy + Clean(50K) & \textbf{82.09} \\
\bottomrule
\end{tabular}
\end{center}
\label{tab:Clothing1M}
\end{table}
\subsection{Comparison on Clothing1M}
For the comparison on Clothing1M, we evaluate our AFM in three different settings following \cite{lee2018cleannet,patrini2017making,zhang2019metacleaner,han2019deep}: (1) only the noisy set are used for training, (2) the 25K extra manually-verified images~\cite{lee2018cleannet} are added into the noisy set for training, and (3) the 50K clean training images are added into the noisy set.

The comparison results are shown in Table \ref{tab:Clothing1M}. For the first setting, our AFM improves the baseline method from 68.94\% to 74.12\%, and consistently outperforms MetaCleaner, JointOptim, and SMP (Initial) by about 2\%. Although 
SMP (Final) performs on par with AFM in this setting, it needs several training-and-correction loops and careful parameter tuning. Compared to SMP (Final), our AFM is simpler and almost free in computational cost.

For the second setting, other methods except for MetaCleaner mainly apply the 25K verified images to train an accessorial network~\cite{lee2018cleannet,patrini2017making} or to select the class prototypes~\cite{han2019deep}. Following \cite{zhang2019metacleaner}, we train our AFM on 1M noisy training set, and then fine-tune it on the 25K verified images. As shown in Table \ref{tab:Clothing1M}, AFM obtains 77.21\% which sets new record in this setting. Specifically, our AFM is better than MetaCleaner and SMP by 1.21\% and 0.77\%, respectively.

For the third setting, all the methods first train models on the noisy set and then fine-tune them on the clean set. CurriculumNet~\cite{guo2018curriculumnet} uses a deeper CNN backbone and obtains accuracy 81.5\%, which is slightly better than SMP and other methods. Our AFM  outperforms CurriculumNet by 0.59\%, and is better than MetaCleaner by 1.31\%. It is worth emphasizing that both CurriculumNet and SMP need to train repeatedly after model convergence which are complicated and time-consuming, while AFM is much simpler and almost free.


\begin{table}[tp]
\caption{Results of different feature interactions in Group-to-Attend module. $^*$It removes $FC_a$ and $FC_b$ in GA module.}
\begin{center}
\begin{tabular}{@{}cccccccc@{}}
\toprule
\# & Interaction type & Training Data  & Acc. (\%)\\ \midrule
1 & Concatenation & Food101N & 86.95 \\
2 & Concatenation$^{*}$ & Food101N & 86.51 \\
3 & Sum & Food101N & \textbf{87.23} \\
4 & Sum$^{*}$ & Food101N & {86.12} \\
5 & Multiplication & Food101N & 86.64 \\
\bottomrule
\end{tabular}
\end{center}
\label{tab:interactions}
\end{table}

\begin{table}[t]
\centering
\makebox[0pt][c]{\parbox{1\textwidth}{%
    \begin{minipage}[b]{0.5\hsize}\centering
 \caption{Evaluation of trade-off $\lambda$.}
\resizebox{\linewidth}{!}{
\begin{tabular}{@{}cccccccc@{}}
\toprule
$\lambda$ & 0.00  & 0.25 & 0.50 & 0.75 & 1.00\\ \midrule
Acc. (\%) & 84.51 & 86.75 & 86.97 & \textbf{87.23} & 86.47  \\
\bottomrule
\end{tabular}}
\label{tab:fcratio}
    \end{minipage}
    \hfill
    \begin{minipage}[b]{0.45\hsize}\centering
        	\caption{Evaluation of group size.}
	\resizebox{\linewidth}{!}{
	\begin{tabular}{@{}ccccccccccc@{}}
		\toprule
		Size & 2  & 3 & 4 & 5 & 6  \\ \midrule
		Acc. (\%) & \textbf{87.23}  & 86.46 & 86.01 & 85.92 & 85.46 \\ \midrule
	\end{tabular}}
	\label{tab:numberinteraction}
    \end{minipage}
}}
\end{table}

\subsection{Ablation Study}
\textbf{Evaluation of feature interaction types}. 
\textit{Concatenation}, \textit{sum} and \textit{element-wise multiplication} are three popular feature fusion or interaction methods. MetaCleaner~\cite{zhang2019metacleaner} simply takes the \textit{concatenation}, and ignores the impact of the interaction types.
We conduct an ablation study for them along with the projection in Group-to-Attend module. Specifically, the group size is set to 2 for this study. Table \ref{tab:interactions} presents the results on Food101N. Two observations can be concluded as following. First, with $FC_a$ and $FC_b$, the \textit{sum} interaction consistently performs better than the others. Second, for both \textit{concatenation} and \textit{sum}, it is better to use the projection process. As mentioned in Section 3.2, removing $FC_a$ and $FC_b$ leads to random attention weights for \textit{sum} interaction, which may degrade our AFM to standard Manifold mixup~\cite{zhang2017mixup}. Nevertheless, it still improves the baseline (\ie 84.51\%) slightly.


\textbf{Evaluation of the trade-off weight $\lambda$}. 
In training phase, $\lambda$ is used to trade-off the loss ratio between $\mathcal{L}_{afm}$ and $\mathcal{L}_{org}$. We evaluate it by increasing $\lambda$ from 0 to 1 on Food101N, and present the results in Tabel \ref{tab:fcratio}. 
We achieve the best accuracy with default $\lambda$ (\ie 0.75). Decreasing $\lambda$ means to use less interpolations from AFM, which gradually degrades the final performance. Particularly, $\lambda=0$ is our baseline that only uses original noisy training data. In the other extreme case, using only the interpolations from AFM is better than the baseline but slight worse than the default one. This may be explained by that the massive interpolations are more or less smoothed by our AFM since the interpolation weights cannot be zeros due to the GA module. Hence, adding original features can be better since these features fill this gap naturally.
%

\textbf{Evaluation of the group size.}
Our previous experiments fix the group size as 2 which construct pairwise samples for generating virtual feature-target vectors. Here we explore different group sizes for our attentive feature mixup, Specifically, we increase the group size from 2 to 6, and present the results in Table \ref{tab:numberinteraction}.
As can be seen, enlarging the group size gradually degrades the final performance. This may be explained by that large group size interpolates over-smoothed features which are not discriminative for any classes. 

\textbf{Intra-class mixup vs. Inter-class mixup}. 
To investigate the contributions of intra-class mixup and inter-class mixup, we conduct an evaluation by exploring different ratios between intra- and inter-class interpolations with group size 2. Specifically, we constrain the number of interpolations for both mixup types in each minibatch with 8 varied ratios from 10:0 to 2:8 on the Food101N dataset. The results are shown in Figure \ref{fig:ratio}. 
Several observations can be concluded as following. First, removing the inter-class mixup (\ie 10:0) degrades the performance (it is similar with MetaCleaner~\cite{zhang2019metacleaner}) while adding a small ratio (\eg 8:2) of inter-class mixup significantly improves the final result. This indicates that the inter-class mixup is more useful for better feature learning. Second, increasing the ratio of inter-class mixup further  boosts performance but the performance gaps are small. Third, we get the best result by random selecting group-wise samples. We argue that putting constraints on the ratio of mixup types may result in different data distribution compared to the original dataset while random choice avoids this problem.

\textbf{AFM for learning from small dataset}. 
Since AFM can generate numerous of noisy-reduced interpolations in training stage, we intuitively check the power of AFM on small datasets. To this end, we construct sub-datasets from Food101N by randomly decreasing the size of Food101N to 80\%, 60\%, 40\%, and 20\%. The results of our default AFM on these synthetic datasets are shown in Figure \ref{fig:augmentation}. Several observations can be concluded as following. First, our AFM consistently improves the baseline significantly. Second, the improvements from data size 40\% to 100\% are larger than that of 20\%. This may be because that small dataset leads to less diverse interpolations. Third, we interestingly find that our AFM already obtains the state-of-the-art performance with only 60\% data on Food101N.

%

\begin{figure}[t]
\begin{minipage}[b]{0.45\linewidth}
\centering
\includegraphics[width=1\linewidth]{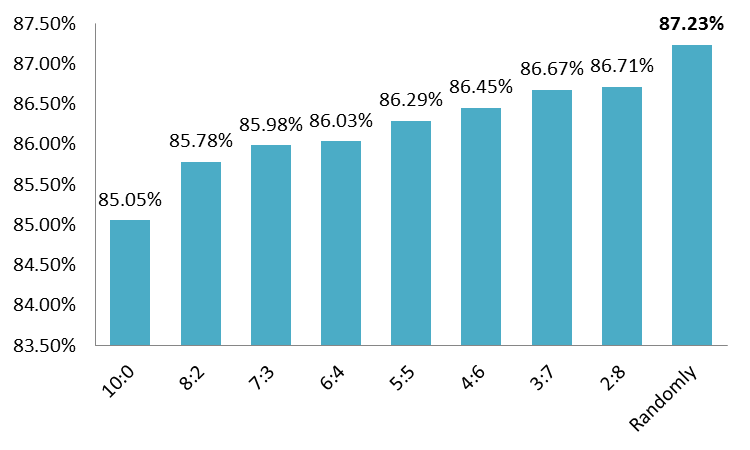}
	\caption{Evaluation of the ratios of Intra- and Inter-class mixup.}
	\label{fig:ratio}
\end{minipage}
\hfill
\begin{minipage}[b]{0.45\linewidth}
\centering
\includegraphics[width=1\linewidth]{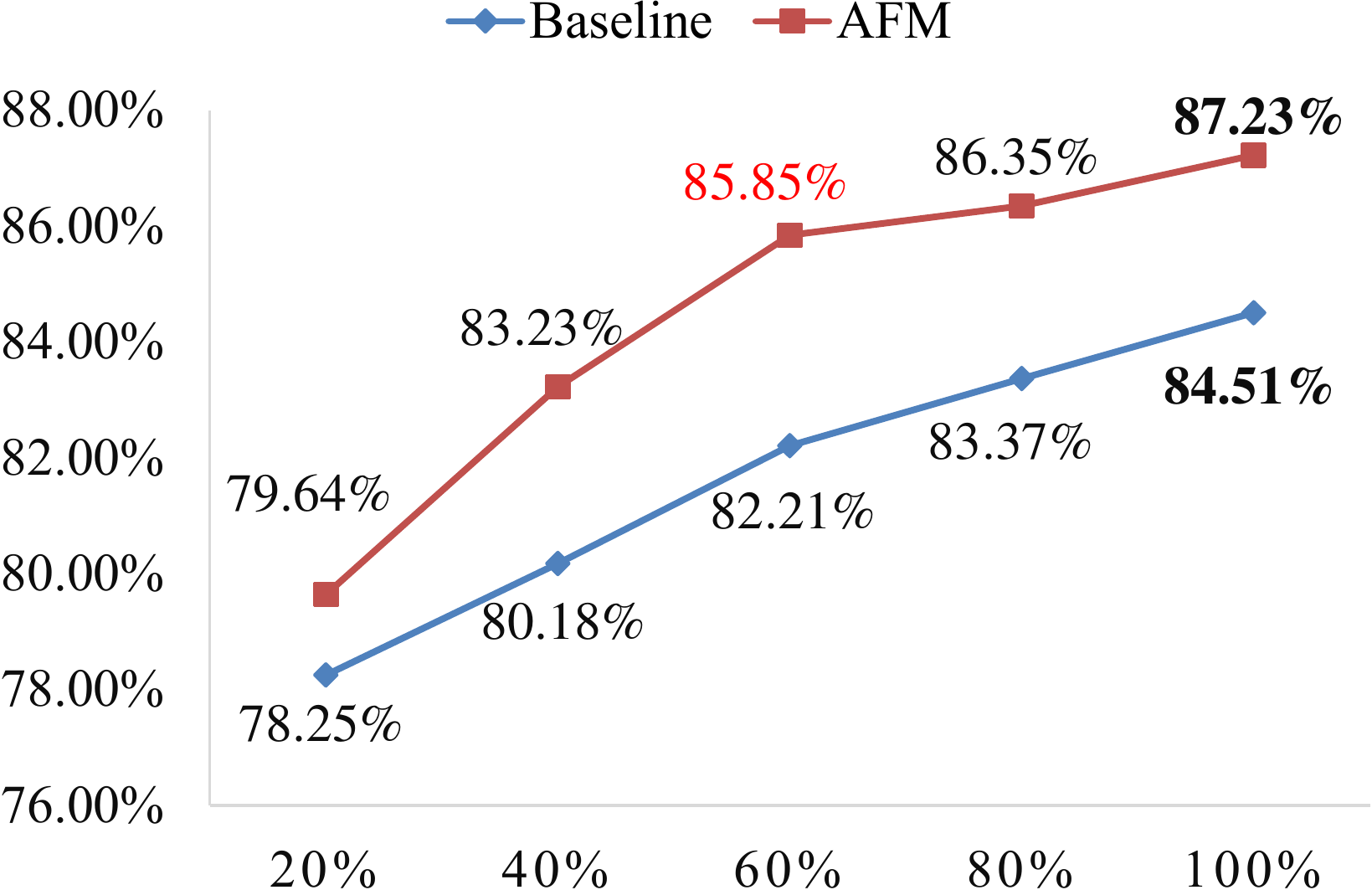}
	\caption{Evaluation of AFM on synthetic small datasets.}
	\label{fig:augmentation}
\end{minipage}
\end{figure}
%

\begin{table}[t]
\begin{minipage}{0.45\linewidth}
\caption{Comparison of our AFM with mixup \cite{zhang2017mixup} and Manifold mixup \cite{verma2018manifold}. We also evaluate $f_{c1}$ and $f_{c2}$ for them.}
\resizebox{\linewidth}{!}{
\begin{tabular}{@{}llccccc@{}}
\toprule
 Method & $f_{c1}$ + $f_{c2}$ &$f_{c1}$ + $f_{c2}$ (Shared) \\ \midrule
  mixup \cite{zhang2017mixup} & 85.36\%  & 85.63\% \\
 Manifold mixup \cite{verma2018manifold} &  85.85\% & 86.12\% \\
 AFM (Ours) & \textbf{86.97}\% & \textbf{87.23\%}\\
\bottomrule
\end{tabular}}
\label{tab:ablations}
\end{minipage}
\hfill
	\begin{minipage}{0.5\linewidth}
		\centering
		\includegraphics[width=0.8\linewidth]{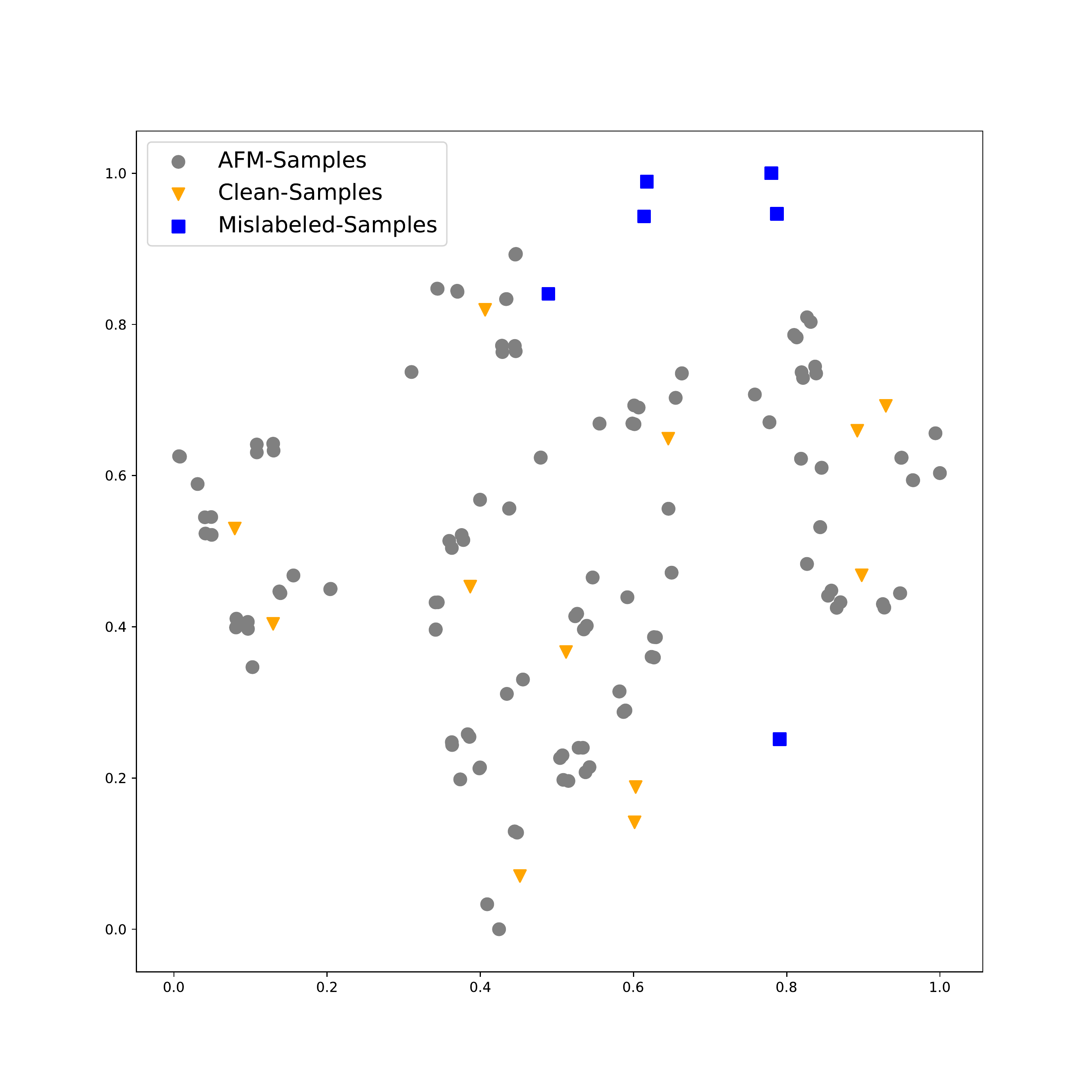}
		\captionof{figure}{AFM sample distribution.}
		\label{fig:sampledis}
	\end{minipage}
\end{table}

\begin{figure}[t]
	\includegraphics[width=0.92\textwidth]{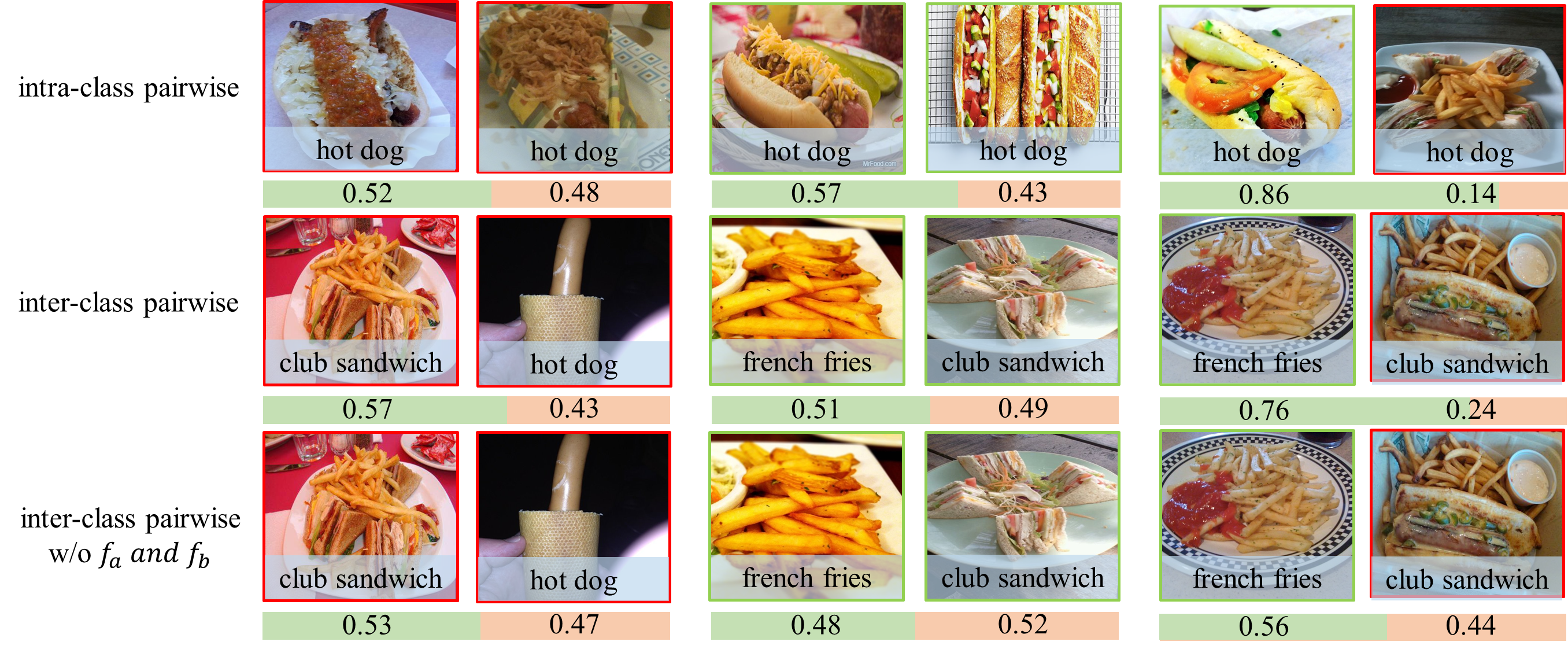}
	\caption{Visualization of the attention weights in our AFM. The green and red boxes represent the clean and noisy samples. }
	\label{fig:visualization}
\end{figure}

\textbf{AFM vs. classic mixup}. 
Since our AFM is related to the mixup scheme, we compare it to the Standard mixup \cite{zhang2017mixup} and Manifold mixup \cite{verma2018manifold}. The Standard mixup \cite{zhang2017mixup} interpolates samples in image level while Manifold mixup \cite{verma2018manifold} in feature level. Both of them drawl the interpolation weights randomly from a $\beta$ distribution. Our method introduce a Group-to-Attend (GA) module to generate meaningful weights for noise-robust training. As the new interpolations and the original samples can contribute differently, we separately apply classifiers for them, \ie $f_{c1}$ for interpolations and $f_{c2}$ for original samples. Table \ref{tab:ablations} presents the comparison. Several observations are concluded as following. First, for both classifier setting, our AFM outperforms the others largely, \eg AFM is better than standard mixup by 1.6\% and the Manifold mixup by 1.11\% in the shared classifier setting. Second, the shared classifiers are slightly better than the independent classifiers for all methods, which may be explained by that sharing parameters makes the classifier favor linear behavior over all samples thus reducing over-fitting and encouraging the model to discover useful features.

\subsection{Visualizations}
To better investigate the effectiveness of our AFM, we make two visualizations: i) attentive mixup sample distribution between clean and noisy samples in Figure \ref{fig:sampledis} and ii) the normalized attention weights in Figure \ref{fig:visualization}. For the former, we randomly select several noisy samples and clean samples on the VK(25) set of Food101N and apply our trained AFM model to generate virtual samples (\ie AFM samples), and then use t-SNE to visualize all the real samples and attentive mixup samples. Figure \ref{fig:sampledis} evidently shows that our AFM samples are mainly distributed around the clean samples, demonstrating our AFM suppresses noisy samples effectively. \textit{It is worth noting that classical mixup samples are doomed to distribute around all the real samples rather than only clean samples}.

For the latter visualization, the first row of Figure \ref{fig:visualization} shows three types of pairs for the intra-class case, the second row for the inter-class case, and the third row for the inter-class case without projection in the Group-to-Attend module. The first column denotes the ``noisy+noisy'' interpolations, the second column denotes ``clean+clean'', and the third column denotes ``clean+noisy''. Several finds can be observed as following.
First, for both intra- and inter-class cases, the weights of ``noisy+noisy'' and ``clean+clean'' interpolations trend to be equal since these interpolations may lie in the decision boundaries which make the network hard to identify which is better for training.
Second, for the ``clean+noisy'' interpolations on the first two rows, our AFM assigns evidently low weights to these noisy samples which demonstrates the effectiveness of AFM.
Last, without projection in the Group-to-Attend module, our default AFM loses the ability to identify noisy samples as shown in the last image pair.